\documentclass{article}

\PassOptionsToPackage{numbers, compress}{natbib}
\PassOptionsToPackage{numbers,compress}{natbib}

\usepackage[preprint]{neurips_2022}

\usepackage[utf8]{inputenc} % allow utf-8 input
\usepackage[T1]{fontenc}    % use 8-bit T1 fonts
\usepackage{hyperref}       % hyperlinks
\usepackage{url}            % simple URL typesetting
\usepackage{booktabs}       % professional-quality tables
\usepackage{amsfonts}       % blackboard math symbols
\usepackage{nicefrac}       % compact symbols for 1/2, etc.
\usepackage{microtype}      % microtypography
\usepackage{xcolor}         % colors
\usepackage{subfigure}
\usepackage{authblk}

\usepackage{amsthm,amsmath,amssymb}
\usepackage{algorithm,algorithmic}
\usepackage{graphicx}
\usepackage{multirow}

\newcommand{\E}{\mathbb{E}}
\newcommand{\A}{\mathcal{A}}
\newcommand{\D}{\mathcal{D}}
\renewcommand{\S}{\mathcal{S}}

\newtheorem{theorem}{Theorem}
\newtheorem{lemma}{Lemma}
\newtheorem{definition}{Definition}

\title{Learning to Prove Trigonometric Identities}

\author{Zhou Liu$^{1*}$, Yujun Li$^{1*}$, Zhengying Liu$^1$, Lin Li$^2$, Zhenguo Li$^1$\\
$^{1}$Huawei Noah's Ark Lab ~~~ $^{2}$HiSilicon Turing Department \\

\{liuzhou15,liyujun9,liuzhengying2,lilin29,Li.Zhenguo\}@huawei.com
}
\begin{document}
\thanks{Both authors contributed equally to this research.}
\maketitle

\begin{abstract}

Automatic theorem proving with deep learning methods has attracted attentions recently. In this paper, we construct an automatic proof system for trigonometric identities. We define the normalized form of trigonometric identities, design a set of rules for the proof and put forward a method which can generate theoretically infinite trigonometric identities. Our goal is not only to complete the proof, but to complete the proof in as few steps as possible. For this reason, we design a model to learn proof data generated by random BFS (rBFS), and it is proved theoretically and experimentally that the model can outperform rBFS after a simple imitation learning. After further improvement through reinforcement learning, we get AutoTrig, which can give proof steps for identities in almost as short steps as BFS (theoretically shortest method), with a time cost of only one-thousandth. In addition, AutoTrig also beats Sympy, Matlab and human in the synthetic dataset, and performs well in many generalization tasks.

\end{abstract}  
\section{Introduction}
In recent years, deep learning has made significant progress in the field of automatic theorem proving (ATP) \cite{Wu2021INTAI,Wang2020LearningTP,Polu2020GenerativeLM}. However, it is still a challenging issue to solve general ATP problems, usually involving complicated symbolic deduction processes, data scarcity and generalization challenge. In this paper, we construct an automatic proof system and 
design an automated prover AutoTrig for trigonometric identities. AutoTrig can prove a trigonometric identity by gradually converting the difference between the two sides of the equation to 0, step by step. We focus on automatically proving trigonometric identities due to the following reasons: 
\begin{itemize}
    \item This problem embodies the general difficulty of automatic theorem proving. Our framework for proving trigonometric identities can be extended to solve other theorem proving problems. 
    \item Data sparsity is a key challenge for deep learning model on theorem proving. In this problem, the issue can be addressed by generating a large number of trigonometric identities.
    \item The form of trigonometric identities is rich and changeable, so it is convenient to study model's generalization ability.
    \item This problem is moderately difficult and does not require advanced mathematical background knowledge. Thus, it facilitates drawing on expert experience to develop algorithms and understanding model's proof. 
\end{itemize}

% 还可以补充几点 1.三角函数问题可以依靠Sympy等库的专业支持来解决符号处理问题中常见的合并同类项等基本操作，使我们的精力主要放在如何采取合适的操作来逐步变换以最终实现证明上 

% 2.数据稀缺问题,
% 3.三角函数的形式丰富、复杂度丰富，既有简单的问题，也有复杂的问题。我们可以学习一些产生的数据D，然后调整参数产生新的数据分布D*，来研究泛化性问题

% 4.三角函数的难度适中，有利于引入专家经验
 
Theorem proving usually requires matching underlying symbol processing tools. Some works use professional languages to formalize theorems \cite{Polu2020GenerativeLM,Han2021ProofAC,Wang2017PremiseSF}. These languages are very expressive and can describe many different types of theorems. But they also require a high level of mathematical literacy for users to understand the final proof. Taking advantage of the relatively limited scope of trigonometric identities, we use Sympy as our underlying symbol processing tool, which can help to complete some of the most basic mathematical operations such as expanding brackets, merging common items, and so on \cite{Sympy}. Sympy is a Python-based library that delivers a simple and intuitive presentation of trigonometric functions in the form of strings.
So we can get rid of the complexity of symbolic processing and focus on decision inference. In addition, due to the small size of Sympy, our final proof system is expected to be lightweight and swift.

Due to the difficulty and the expensiveness to get theorems data from humans, many works use synthetic datasets to overcome data scarcity \cite{Piotrowski2018ATPboostLP,urban2020first,Wang2020LearningTP}. Some works collect proof data by reversing the steps of the generating process \cite{Wu2021INTAI}. But this method also introduces the problem that the inverse is not always the shortest proof. In our work, our goal is not only to complete the proof, but also to complete the proof in as few steps as possible. Considering that most proof problems can be viewed as search problems \cite{Nawaz2019ASO,Kaliszyk2018ReinforcementLO,holden2021machine}, it is best to use BFS to generate proof data. Nevertheless, the complexity of the search space grows exponentially as the number of proving steps increases. The time cost and space cost of BFS is incredibly unacceptable. Hence, we propose rBFS, which reduces the search time by randomly pruning BFS. We prove theoretically and experimentally that, model can outperform rBFS after a simple imitation learning, through learning the proof generated by rBFS. This approach can also be applied to other problems to get shorter proofs in the ATP field. 

Models are always expected to solve new and unfamiliar problems, so the ultimate generalization ability of models is very important \cite{Wu2021INTAI}. We design a method for generating trigonometric identities, which can adjust the distribution of generated data by changing several parameters. We generate a dataset and train a model AutoTrig based on it. Then, we generate several new datasets with different distributions. By testing AutoTrig's performance on these new datasets, we are able to evaluate its generality. We find AutoTrig can still deal with an identity even if we change the order of its terms or change its coefficients. Also, we test AutoTrig's performance on some real problems and find that it still performs well. 

In addition, the interpretability of models has always been a difficulty in deep learning, \cite{cheng2020research,li2021interpretable,zhang2018visual}. For this reason, we make a user study to carefully study the differences between the proofs given by AutoTrig and various methods. It is found that AutoTrig can take appropriate actions to make it possible to merge similar terms or use common factors, leading to a shorter proof. This suggests that AutoTrig has indeed learned some mathematical tricks to shorten proof length.

Our contributions mainly include the following parts: 
\begin{itemize}
    % \item We design a very compact but powerful set of rules, which effectively narrows the search space of the proof trajectory.
    \item We construct an automatic proof system for trigonometric identities, including defining the normalized form, designing
    the rule set, generating identities and their proof.
    % \item We design a pretrained loss according to the chain structure of trigonometric equation proving process. The chain structure widely occurs in automatic theorem proving. This method can be extended to other theorem proving problems. 
    % \item We propose a hierarchical reinforcement learning method to learn trigonometric formula selection. This method can intelligently prove trigonometric equations in fewer steps. 
    \item We point out that imitation learning model trained on the proof generated by rBFS can even give shorter proofs than rBFS, which is proved theoretically and experimentally. This method can be extended to other theorem proving problems. 
    \item After reinforcement learning, our final model AutoTrig achieves a similar performance to BFS in terms of proof length, with a time cost of only one-thousandth. 
    \item Our final model AutoTrig beats Sympy and Matlab in the synthetic dataset.
    \item We conduct extensive experiments to present the model's generalization ability and a user study to demonstrate the math reasoning ability learned by AutoTrig. 

    % \item Based on the classification model obtained in the training data, we used it as the initial model of the reinforcement learning policy network and trained a value function model on the training set. Use reinforcement learning algorithms to further improve the efficiency of solving classifiers for the model. 
\end{itemize}

\section{Related work}
\paragraph{Symbolic math toolbox}
Trigonometric identities can be verified by symbolic computation, which can simplify the difference between the two sides of the equation to zero. Note that there are a few popular related tools including Matlab \cite{MATLAB:R2021a} and Sympy \cite{Sympy}. Matlab is a mature commercial software with a strong ability in numerical and symbolic computation. Sympy is an open-sourced libary which adopt a heuristic method to process trigonometric expressions \cite{Sympy,Fu2006AutomatedAR}.  Both of them can be used to verify trigonometric identities based on their built-in simplification functions.

\paragraph{Learning-based theorem proving}
Deep learning techniques have been widely applied to ATP. Graph neural networks (GNNs) \cite{Scarselli2009TheGN,Daigavane2021UnderstandingCO} and transformers \cite{Vaswani2017AttentionIA} are two commonly used networks to extract features from mathematical theorems. The hierarchical structure of mathematical expressions can be viewed as trees, so that GNNs are employed in many works \cite{Wang2020LearningTP,Paliwal2020GraphRF,Whalen2016HolophrasmAN,Aygun2020LearningTP}. Besides, mathematical expressions can also be regarded as sequences of tokens. Thanks to the strong expressive ability of transformers in natural language processing, many works use transformers to extra features from formulas \cite{Polu2020GenerativeLM,Lample2020DeepLF,Saxton2019AnalysingMR}. There are three learning paradigms in ATP: supervised learning, generative learning, and reinforcement learning. Supervised learning methods are used to learn the mapping from the proof state to the proof tactic \cite{Aygun2020LearningTP,Rocktschel2017EndtoendDP,Wang2017PremiseSF,Crouse2019ImprovingGN}. Generative learning approaches are used to generate intermediate proof steps for proof goals \cite{Polu2020GenerativeLM,Wang2020LearningTP,Polu2022FormalMS}. \cite{Kaliszyk2018ReinforcementLO,Wu2021TacticZeroLT} model the proving processes by Markov decision processes, and use reinforcement learning methods to solve it.

\paragraph{Datasets for theorem proving}
There are many datasets in the field of ATP. These datasets can be divided into two categories: manually formalized datasets and synthetic datasets. \cite{Community2020TheLM} presents MathLib, a unified library of mathematics built with a community-driven effort. \cite{Zheng2021MiniF2FAC} provides a dataset of formal Olympiad-level mathematics problems. \cite{Li2021IsarStepAB} builds the largest non-synthetic dataset from the largest repository of proofs written by human experts in a theorem prover. Manually collected datasets have a small number of samples on each particular kind of problem. There are also some works to artificially synthesize datasets. ATPboost generates negative samples for a binary classification task to estimate the pairwise-relevance of (theorem, premise) pairs \cite{Piotrowski2018ATPboostLP}. MetaGen uses neural generators to generate large amounts of data samples for supervised learning \cite{Wang2020LearningTP}. INT provides a set of manually designed rules to synthesize datasets \cite{Wu2021INTAI}.

\section{Preliminary}
In trigonometry, trigonometric identities are equalities that involve trigonometric functions. They are always true, regardless of the values of the variables involved. $\sin(\cdot)$ and $\cos(\cdot)$ are the two major trigonometric functions which can represent other trigonometric functions like $\tan(\cdot)$, $\cot(\cdot)$, $\csc(\cdot)$, $\sec(\cdot)$. Trigonometric identities have a variety of forms, but we can always manipulate them into polynomial form. In this paper, we concentrate on the trigonometric identities defined in the following normalized form:
\begin{equation}\label{eq:trigonometric_expression}
\begin{aligned}
    & \text{left expression} = \text{right expression}\\
    & \textbf{Normalized form:    } \text{left expression} - \text{right expression} = 0\\
    & \text{expression} = \sum_i \text{term}_i \\ 
    & \text{term}_i = \prod_j a_j \sin^{b_j}(c_j x + d_j) \cos^{e_j}(f_j x + g_j),
\end{aligned}
\end{equation}
where $a_j, b_j, c_j, e_j, f_j$ are integers, $d_j, g_j$ are angles related to $\pi$. We can prove an identity by transforming the difference of the two sides to zero. Table \ref{tab:rule} lists 8 basic theorems of trigonometric identities which can prove various theorems \citep{andreescu2004103}. Our goal is to prove trigonometric identities via these rules within the least steps.

\begin{table*}[h]\small
\centering
\caption{Trigonometric transformation rules.}
\label{tab:rule}
\begin{tabular}{|c|c|c|}
\hline
Type &Form &Abbreviation  \\ \hline
\multirow{4}{*}{Product-to-sum formulas}           
& $\cos \alpha \cos \beta ={\cos(\alpha -\beta )/2+\cos(\alpha +\beta )/2}$ &$P_{cc}$\\ \cline{2-3} 
& $\cos \alpha \sin \beta ={\sin(\alpha +\beta )/2-\sin(\alpha -\beta )/2}$ &$P_{cs}$\\ \cline{2-3}
& $\sin \alpha \cos \beta ={\sin(\alpha +\beta )/2+\sin(\alpha -\beta )/2}$ &$P_{sc}$\\ \cline{2-3} 
& $\sin \alpha \sin \beta ={\cos(\alpha -\beta )/2-\cos(\alpha +\beta )/2}$ &$P_{ss}$\\ \hline
\multirow{4}{*}{Addition and subtraction formulas} 
 
& $\cos(\alpha + \beta) = \cos(\alpha) \cos(\beta) - \sin(\alpha) \sin(\beta) $                     &$A_{c+}$\\ \cline{2-3}
& $\cos(\alpha - \beta) = \cos(\alpha) \cos(\beta) + \sin(\alpha) \sin(\beta) $                     &$A_{c-}$\\ \cline{2-3}
& $\sin(\alpha + \beta) = \sin(\alpha)\cos(\beta) + \cos(\alpha) \sin(\beta) $                      &$A_{s+}$\\ \cline{2-3} 
& $\sin(\alpha - \beta) = \sin(\alpha)\cos(\beta) - \cos(\alpha) \sin(\beta) $                      &$A_{s-}$\\ 

\hline
\end{tabular}
\end{table*}

\begin{figure}[tb!]
    \centering
    \includegraphics[width=0.8\textwidth]{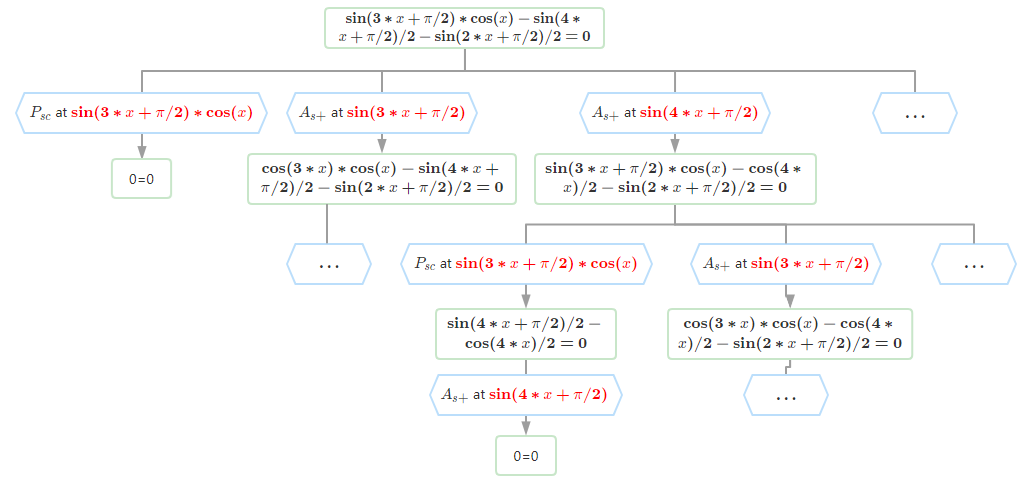}
    \caption{The proof tree of trigonometric identity $\sin(3*x+\pi/2)*\cos(x) - \sin(4*x+\pi/2)/2 - \sin(2*x+\pi/2)/2 = 0$.}
    \label{fig:proof_tree}
\end{figure}

The proving problem can be viewed as a search problem \cite{Nawaz2019ASO,Kaliszyk2018ReinforcementLO,holden2021machine}. As shown in Figure \ref{fig:proof_tree}, the root node represents a trigonometric identity problem. It can select a sequence of transformation actions to promote the proof until we get $0 = 0$. In Figure \ref{fig:proof_tree}, a green box represents a trigonometric expression $s$ which is equal to zero. A blue box is a transformation action $a$ which indicates a specific transformation rule in Table \ref{tab:rule} on a specific object of $s$. We can see that the problem can be proved within one step by choosing the action to apply rule $P_{sc}$ on $\sin(3*x+\pi/2)*\cos(x)$, while it needs more proof steps to complete the proof if choosing other actions. The goal is to find a short trajectory $\{(s_i, a_i)\}_{i=1}^N$ to finish the proof.

% \subsection{Shortcomings of current method}
% Although $Naive Method$ is simple and intuitive, the weakness is that the steps sometimes will be quite long. An example is listed in XX. In fact, we can get a shorter trajectory, listed in XX. As is shown in Fig 1, we actually have two kinds of rules to choose from at each step, and for each kind, we may have multiple objects on which this kind of rule can be applied. For $Naive Method$, we actually follow two principle. The first principle is we prefer to apply P-rules rather than S-rules, when both of them can be applied on the $expr$. The second principle is we always choose the first object when there are more than one objects which can be applied rules on. The case above has shown us the need to abandon these two principles for fewer steps. Naturally, a simple way to get the shortest trajectory is Breadth First Search (BFS). However, the time complexity is exponential, because we can take many action for each step. The problem is that we need to find a new method to get a short trajectory with small time cost.

\section{Methodology}
At first, we use a random method to generate a lot of trigonometric identities and use rBFS to generate proof for these identities. Then, an imitation learning method is employed to learn a deep learning model, which is followed by a reinforcement learning method to obtain improvement. The complete process is shown in Figure \ref{fig:structure}. It needs to be pointed out that the model trained on the dataset generated by rBFS can even give a shorter proof than rBFS, only after a simple imitation learning. We give a theoretical proof for this point. 

\begin{figure}[htb!]
    \centering
    \includegraphics[width=0.8\textwidth]{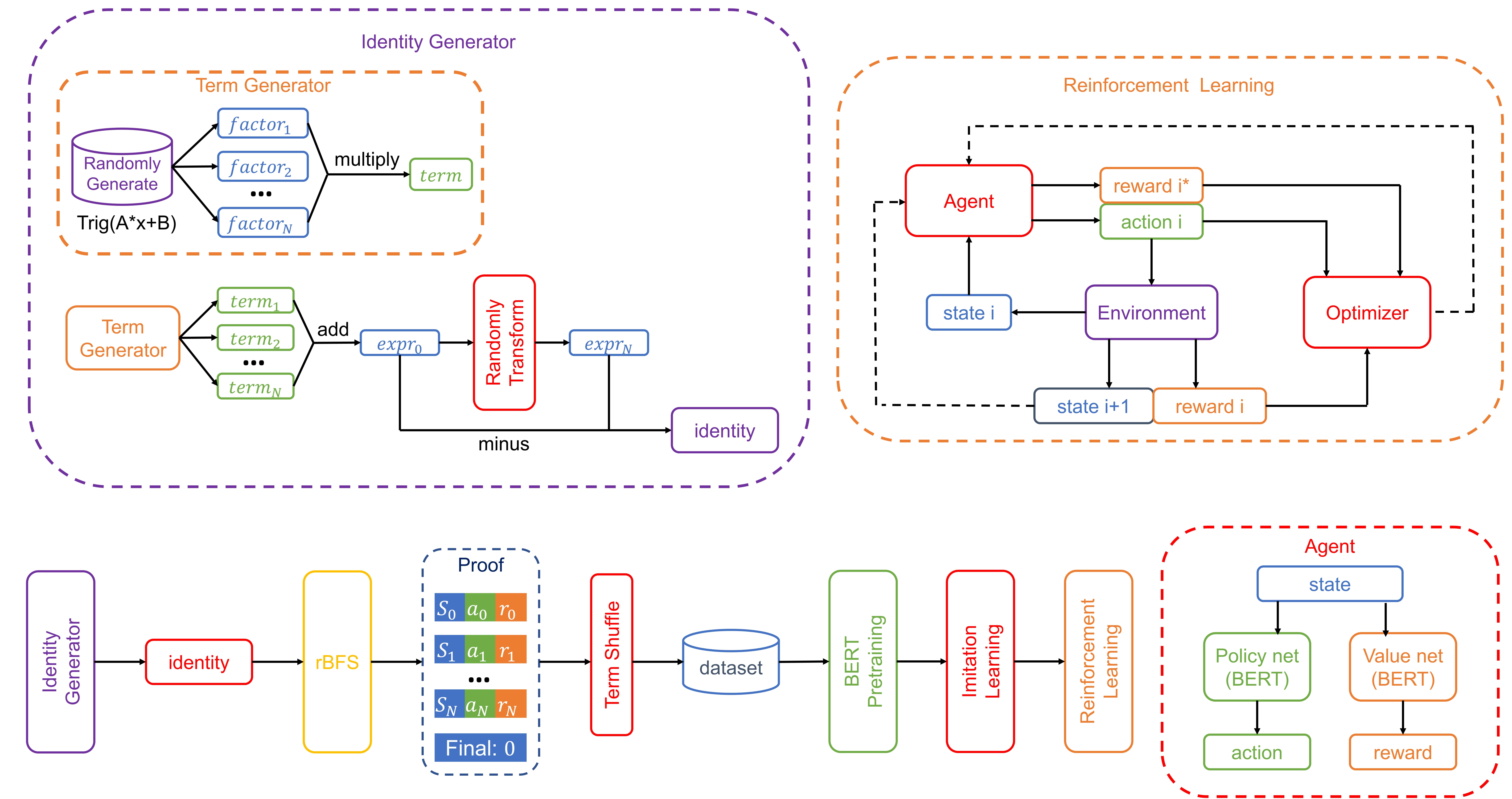}
    \caption{Schematic illustration of our work, including dataset generating, imitation learning, model pretraining and reinforcement learning}
    \label{fig:structure}
\end{figure}

\subsection{Synthetic dataset}

A successful neural network requires a large amount of training data. However it is quite expensive to collect a sufficient number of problems by manual effort. Inspired by the work of INT \cite{Wu2021INTAI}, we propose a bottom-up method to build a dataset of trigonometric identities. It mainly consists of four steps:
\begin{enumerate}
    \item Randomly generate $n$ elements in the form of $\sin(Ax + B)$ or $\cos(Ax + B)$. Multiply these $n$ elements and a  coefficient $C$ to get a term in the form of $C*\prod_{i}^n trig_i(A_ix+B_i)$, where $trig_i$ is selected from $\{ \sin, \cos \}$.
    \item Repeat step 1 for $m$ times and get $m$ terms. Add these terms and get a sum denoted by $E_0$. 
    \item Randomly select several terms and a corresponding rule in Table \ref{tab:rule} to transform the current trigonometric expression $E_0$ and get a new equivalent expression $E_1$.
    \item Repeat step 3 for $t$ times and get an expression $E_t$ which is equivalent to $E_0$. Thus, by subtracting $E_0$ from $E_t$, we can obtain a nontrivial trigonometric identity $E_t - E_0 = 0$.
\end{enumerate}

More details about the generating process are shown in Appendix A. With the above approach, we can generate an unlimited number of trigonometric identities. Note that in the procedure of constructing identities, the transformation rule of each step can be recorded, which can be used for training and learning, just like INT \cite{Wu2021INTAI}. In our work, however, we do not adopt this strategy to get the proof data. Because we notice that the transformation during the generation is random, which means that the reverse transformation sequence may not be the shortest proof. We show a case in the Appendix A to illustrate this situation, which reveals the shortcoming of this method.

In order to get higher-quality (i.e. short, valid) proof data, BFS is the best method, which can derive a solution with the least proving steps by trying all possible rules at each step. However, the computational complexity will grow exponentially as the number of proving steps increases. The Experiment section shows that the average time cost is 166 seconds, which means 100,000 identities will cost more than half a year. Even though with 16 parallel processes, it will take more than one week. To alleviate this issue, we use rBFS instead. The major difference is that BFS tries all possible branches at each node while rBFS tries 3 randomly selected branches. The pairs of node and selected branch $(s_i, a_i)$ along the shortest proof trajectory are collected as the proof data, with more details in Appendix A.

\subsection{Pretraining}
Similar to natural language processing, the context information of trigonometric expressions benefits feature extraction. Recent work has demonstrated that transformer-based neural networks have a great ability to learn representation in ATP \cite{Saxton2019AnalysingMR,Wu2021INTAI,Polu2022FormalMS,Polu2020GenerativeLM,Aygun2020LearningTP}. Thus, we use a popular transformer-based network, BERT \cite{Devlin2019BERTPO}, as the backbone network to extract hidden features from trigonometric expressions. We adopt the key ingredient, masked language modeling, to pretrain BERT on the synthetic data in Section 4.1. Each expression is performed by WordPiece tokenization \cite{Wu2016GooglesNM}. We mask 15\% of all tokens in each expression at random. The objective is to predict the vocabulary id of the masked token based on its context. 
We use a basic BERT which consists of 12 hidden layers with a hidden state size of 768, 12 attention heads and an intermediate dimension of 3072. The token sequence length is set to 256.

\subsection{Learning}
The policy network to select transformation actions is built by adding fully connected networks to the pretrained BERT. When we train the policy network, the parameters of BERT and additional networks are both updated. Based on the synthetic dataset in Section 4.1, the policy network is trained via a typical classification task where $s$ is the input, the one-hot vector of $a$ is the label. It is interesting to see that the simple imitation policy trained on the dataset from rBFS can even get shorter proofs than rBFS.

\begin{definition}
With rBFS exploring the proofs, $L(s)$ denotes the average proof length starting from trigonometric expression $s$, and $L(s, a_i)$ denotes the average proof length after selecting branch $a_i$ starting from node $s$. $L(s')$ denotes the average proof length of $s'$ where $s'$ is transformed from $s$ by selecting transformation action $a$. Then, we have the following recursive definitions,
\begin{align}
    & L(s) = \E_{(a_1, a_2, a_3)} \min \{L(s, a_1), L(s, a_2), L(s, a_3)\} \\
    & L(s, a_i) = L(s') + 1 \\
    & L(s_T) = 0
\end{align}
where $a_1, a_2, a_3$ are three randomly selected branches at node $s$, $s_T$ is the terminal expression $0$.
\end{definition}

\begin{theorem}
Suppose the data from rBFS follows a distribution $(s, a) \sim \D$, where $s$ is an explored trigonometric expression, $a$ is the transformation action to select a branch corresponding to the shortest proof trajectory among three randomly selected branches. The cross entropy is used as the learning objective, 
\begin{align}\label{eq:loss}
    \min_\theta   \E_{(s, a) \sim \D} \left[ - \log P_\theta(a \mid s) \right],
\end{align}
where $P_\theta(\cdot \mid \cdot)$ is the prediction model which gives the probability selecting branch $a$ at node $s$. Then the optimal solution of \eqref{eq:loss} satisfies that $\arg\max_a P_{\theta^*}(a \mid s) = \arg\min_a L(s, a)$.
\end{theorem}

Note that rBFS selects proof steps from three randomly selected branches. The above theorem shows that the imitation model trained on the dataset from rBFS can perform better than rBFS. The intuitive explanation is that rBFS produces many samples that are not completely random. The $\mathtt{min}$ operation allows the sample to have an implicit preference. Deep learning models trained on large amounts of data are able to extract this preference. Therefore, the top one branch selected by the deep learning models are able to give shorter proofs than rBFS. The proof details are shown in Appendix B. This theorem can be extended to other cases where the metric is defined by the operator $\min(\cdot)$ or $\max(\cdot)$.

% \subsection{Reinforcement Learning}

The model obtained from the imitation learning is not optimal since it learns from rBFS which takes a random strategy to finish proofs. As presented in Figure \ref{fig:proof_tree}, this problem can be viewed as a search problem. Thus, we can use reinforcement learning methods to improve the imitation learning model by modeling it with a Markov decision process, with more details in the Appendix C. The reward is set to $1$ when finishing the proof, $-0.1$ otherwise. The objective is to finish the proof with the minimum number of proof steps. This is implemented via an actor-critic algorithm PPO \cite{Schulman2017ProximalPO}. The model from the imitation learning can be used as a warm-start policy. To warm start the value network, it is trained in the dataset from rBFS, 
\begin{align}
  \min_\omega \| V(s_t; \omega) - \sum_i \gamma^i r(s_{t+i}) \|^2,
\end{align}
where $V(\cdot; \omega)$ is the value network, $\omega$ is the network parameter, $r(\cdot)$ is the reward function, $\gamma$ is the discount coefficient. With the initial policy and value network, we use PPO to further update these networks.
% Note that the proof procedure is so unrobust that it will fail as long as one proof step is invalid. In increase the sample efficiency, at each proof step, ten transformation rules are sampled from the policy network. If one selected rule is valid, then we continue the proof procedure.  

\section{Experiment}
In this section, several experiments are performed to validate our proposed methods. We compare our method with several baselines. Further, we compare our method with MATLAB and Sympy. Then, we make a user study to compare our method with humans. Finally, we examine the generalization ability on several new datasets.

\subsection{Experiment Setup}

\subsubsection{Dataset}
With the method introduced in the previous section, we produce about $200,000$ trigonometric identities. About $150,000$ of them are used for model pretraining. rBFS is used to prove these identities, and we get about 600,000 pairs of (trigonometric expression, transformation action). After shuffling addition terms, we finally get about 2,400,000 pairs, which are used for BERT pretraining and imitation learning. 
About $45,000$ of the original identities are used as the initial states in reinforcement learning for generating proof trajectories. 
The $1,000$ of the original identities are used for validation dataset during reforcement training, marked as $D_v$. The last $4,000$ identities are used to evaluate the performance of different methods, marked as $D_t$.

\subsubsection{Evaluation Metrics}
There are three metrics to evaluate a method's performance on proving trigonometric identities: pass rate, average proof length and average time cost. An ideal model is expected to have a high pass rate, a short average proof length and a low average time cost. In addition, due to the limited input length of the BERT network, the number of  terms of each expression in the proof process is constrained to be no more than 8. Otherwise, the transformation action will be regarded as an invalid one. For the sake of fairness, we also set this limitation for other methods. This limitation is also reasonable, since we usually want to get a proof as simple as possible and avoid extremely long expressions during the proving process.

% Sometimes, there is a trade-off between pass rate and average proof length. If a model can only prove simple identities which are usually of short proof length, then it will have a short average proof length but a low pass rate. 

\subsubsection{Baselines}
Several baseline methods are used to compare with our method:  Naive, Filter, BFS, rBFS. Naive method selects a random action among all actions at each step. If the action is invalid, then it fails. However, the random strategy will most likely pick an invalid action. So we add some expert experience to filter invalid actions. Filter method selects a random action among all valid actions at each step. If all actions are invalid, then the search fails. This simple strategy promises a lower time cost, but it can not get a short proof length and a high pass rate. BFS is the most powerful baseline which can complete the proof with the shortest proof length. But its shortcoming is so obvious that it is quite time-consuming to finish the search. rBFS reduces search time by reducing search branches, but also makes the final proof not always the shortest. Considering the influence of randomness, we repeat $10$ times to average the performance of Naive, Filter and rBFS method, in all subsequent tasks.

\subsection{Performance}
\subsubsection{Performance after imitation learning and reinforcement learning}

Figure \ref{fig:foobar}(a) reports the result after imitation learning. Similar to \cite{Polu2022FormalMS}, in each expression it gives $N$ attempts corresponding to the transformation actions of top $N$ prediction probabilities.  
We can find that the model achieves an outstanding performance after imitation training, even if we only take Top-1 policy. As we increase the value of $N$, both the pass rate and average length increase. This reveals that the model with a larger $N$ can solve more complex problems, with time cost increasing by no more than 2\% at the same time. Figure \ref{fig:foobar}(b) shows the performance of model after reinforcement learning and we can observe that the trend of the three indicators changing with N is similar to Figure \ref{fig:foobar}(a). Considering that when N is greater than 5, the improvement of the pass rate is not significant, we adopt the top-5 policy for the subsequent tasks. The model after imitation learning is marked as Model-IL. The model improved by reinforcement learning is marked as AutoTrig. Figure \ref{fig:foobar}(c) shows the model's performance on $D_v$ during the reforcement learning. We can find that the average length decreases rapidly at the beginning and then slowly, with a little increase in pass rate at the same time. This indicates that reinforcement learning further improves the performance of the model effectively.

\subsubsection{Comparison with baselines, Sympy and Matlab}

\begin{figure}
    \centering
    \subfigure[]{\includegraphics[width=0.32\textwidth]{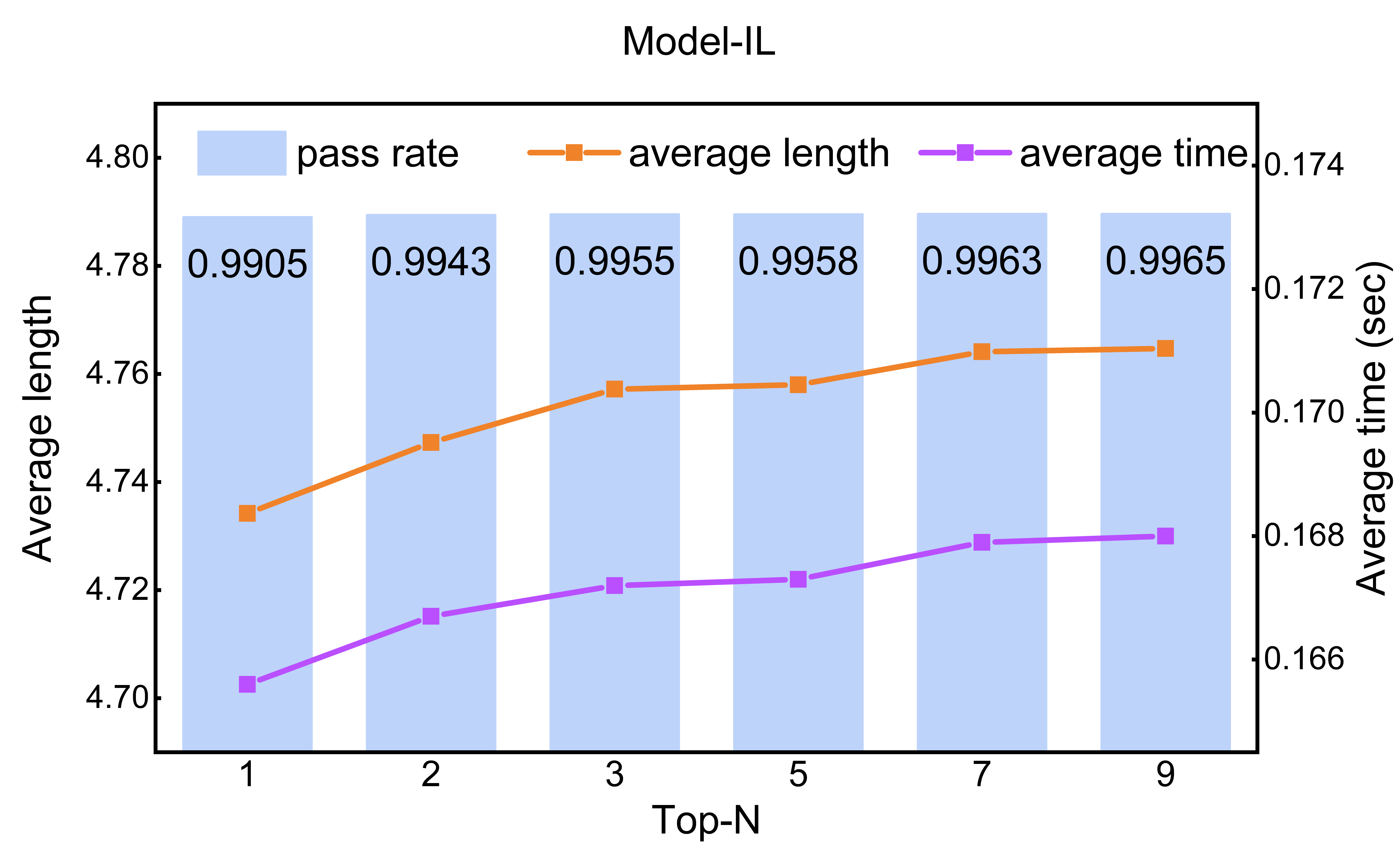}} 
    \subfigure[]{\includegraphics[width=0.32\textwidth]{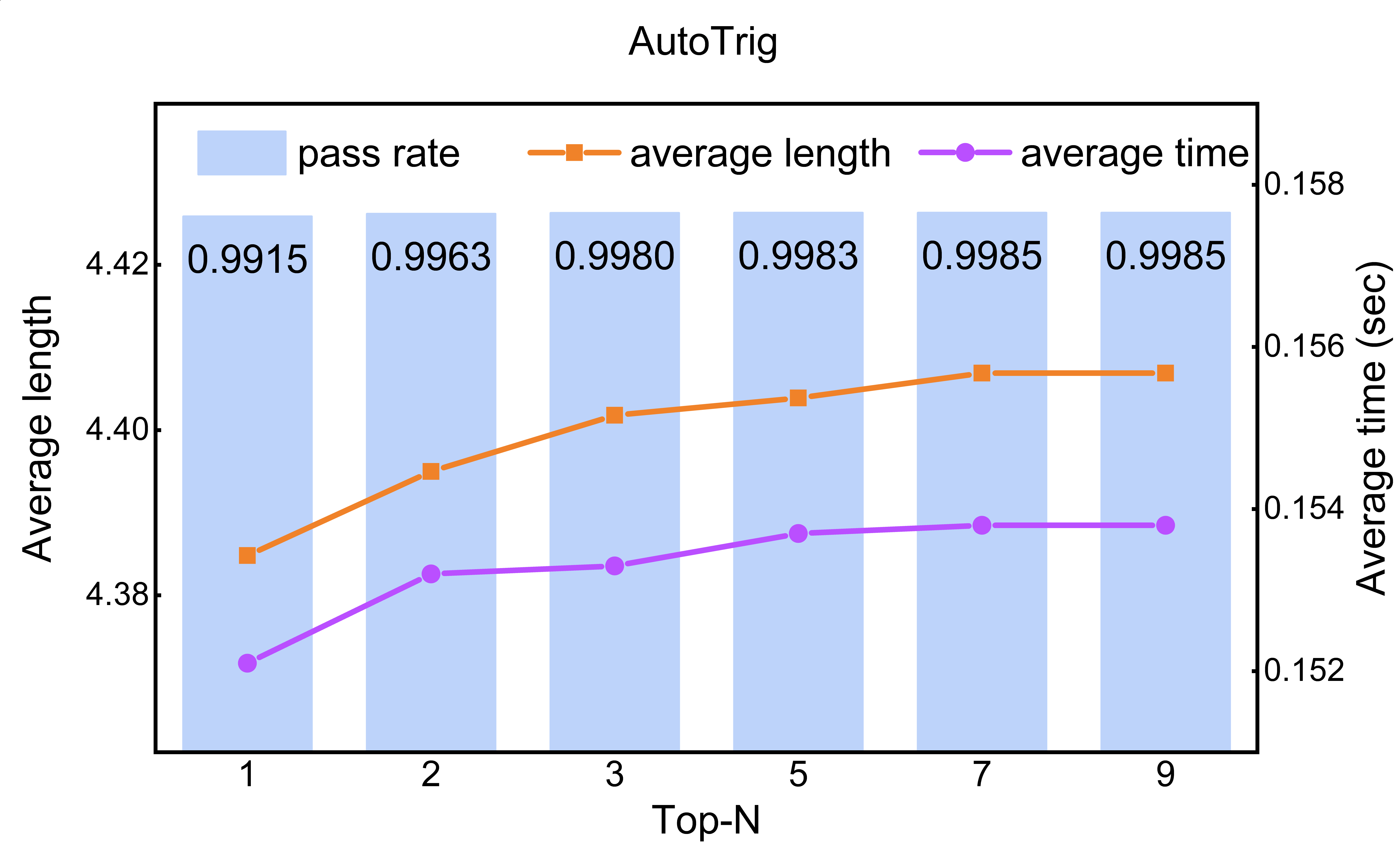}} 
    \subfigure[]{\includegraphics[width=0.32\textwidth]{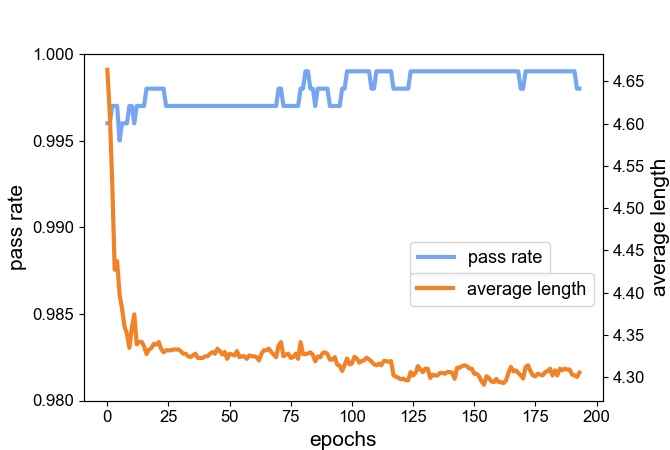}}
    \vspace{-0.1cm}
    \caption{Performance on $D_t$ after imitation learning (a) and reinforcement learning (b). Performance on $D_v$ of AutoTrig during training (c). }
    \label{fig:foobar}
\end{figure}

Table \ref{tab:vs_baselines} demonstrates the performance of different methods on $D_t$. Naive method has a pass rate of 0, without expert experience to filter invalid actions. On the contrast, Filter method can reach a relatively high pass rate of 0.8727. Among all methods, BFS has the highest pass rate and the shortest proof length. This is because it tries all possible proof trajectories. However, its time cost is significantly larger than other methods. Filter method has a small average time cost and a mediocre pass rate but the worst average length among these methods. rBFS achieves a trade-off between average length and average time by randomly cutting off partial searching branches. 

\begin{table}[h]\scriptsize
\centering
\caption{The performance of AutoTrig and different baselines on $D_t$.}
\label{tab:vs_baselines}
\begin{tabular}{|c|c|c|c|c|c|c|c|c|}
\hline
                  method&Naive  & Filter & BFS & rBFS &Model-IL &AutoTrig &Sympy &Matlab\\ \hline
pass rate           &0.0000    &0.8727  &1.0000 &0.9927 &0.9958 &0.9983 &0.9055 &0.9680    \\ \hline
average length &-  &8.39  &4.27  &4.94  &4.76  &4.40  &-  &-\\ \hline
avergate time (sec)   &-     &0.329  &166.141 &17.513   &0.174   &0.159   &0.177   &0.060    \\ \hline
\end{tabular}
\end{table}

As we mentioned earlier, Model-IL is trained on the dataset produced by rBFS. It is interesting to observe that Model-IL gets a competitive pass rate but a obviously smaller average length than rBFS. This is consistent with our theoretical analysis. After further finetuned with reinforcement learning, the pass rate increases slightly, but the average length gets a large improvement. This is because the reward setting of reinforcement learning encourages to produce proofs with shorter length. AutoTrig outperforms rBFS in all three metrics, especially in average length and average time. It's worth noting that AutoTrig achieves similar length of proof as BFS, with a time cost of only one thousandth.
% We can get a conclusion that, AutoTrig has not only learned the ability to select valid actions from all the actions, but also has the ability to select the best actions from all the valid actions to shorten the proof.

% \subsubsection{Comparison with Matlab and Sympy}
Sympy and Matlab are two mainstream symbolic computing tools that can be used to verify trigonometric identities. We compare AutoTrig with them on $D_t$. Since Sympy and Matlab cannot present step-by-step proof trajectories, we only compare pass rate and average time cost here, as shown in Table \ref{tab:vs_baselines}. The fact that Sympy and Matlab can not prove all problems indicates that the synthetic problems are not trivial. In terms of pass rate, AutoTrig performs better than Sympy and Matlab. The open source work about Sympy shows that it is a rule-based symbolic solver \cite{Fu2006AutomatedAR}. This means that AutoTrig can extract non-trivial patterns which cannot be captured by rule-based methods in Sympy. As for average time cost, we can see that Matlab is the fastest among them which takes about one third of the time of AutoTrig. AutoTrig does not use batch processing to speed up procedures, but solves problems one by one. With the acceleration of GPU, AutoTrig can work as fast as Sympy.

\subsection{User Study}
We recruit six volunteers with a master's degree or above to solve this type of problem manually. On account of the expensiveness of volunteers' time, we select 10 problems and compare their performance with baselines and AutoTrig. As is shown in Table \ref{tab:vs_user}, we can find our model obviously outperforms rBFS and volunteers in both average length and average time. To further understand why AutoTrig achieves excellent performance, we show the proof processes of one case in Appendix D. We find that AutoTrig can understand the connections among terms in a expression and take proper actions to rapidly eliminate terms, so that a shorter proof can be attained.

\begin{table}[htb]\scriptsize
\caption{Performance of Volunteers}
\label{tab:vs_user}
\centering
\begin{tabular}{|c|c|c|c|c|c|c|}
\hline
                  method&Naive   & Filter &BFS &rBFS &AutoTrig &Volunteers \\ \hline
pass rate &0.00 &0.62 &1.00 &0.98 &1.00 &1.00   \\\hline
average length &- &14.66 &5.10 &7.26 &5.60 &7.58   \\\hline
average time &- &0.779 &1429.446 &144.460 &0.221 &162.953   \\\hline
\end{tabular}
\end{table}

\subsection{Generalization Study}

\begin{table}[htb!]\scriptsize 
\caption{Performance of Generalization on T1, T2 and T3.}
\label{tab:generalization_task}
\centering 
\begin{tabular}{|c|ccc|ccc|ccc|}
\hline
\multirow{2}{*}{method} & \multicolumn{3}{c|}{T1}     & \multicolumn{3}{c|}{T2}    & \multicolumn{3}{c|}{T3}      \\
                        & pass rate & length & time   & pass rate & length & time  & pass rate & length & time   \\ \hline
Naive &0.0000   &-    &-   &0.000    &-   &-   &0.00  &-  &-\\
Filter                   & 0.8727    & 8.39   & 0.329   & 0.706     & 9.48  & 0.607 & 1.00      & 9.47    & 0.226   \\
BFS                     & 1.0000    & 4.27   & 166.141 & 1.000     & 5.17  & 1115.767  & 1.00      & 5.10    & 2575.008   \\
rBFS                    & 0.9927    & 4.94   & 17.513  & 0.980     & 5.95  & 35.024 & 1.00      & 5.36    & 593.245 \\
AutoTrig                & 0.9988    & 4.40   & 0.154   & 0.945     & 5.85  & 0.275 & 1.00      & 5.20    & 0.207   \\ \hline
\end{tabular}
\end{table}

\paragraph{Shuffled terms (T1)}
When we test model's performance on $D_t$, we get a new expression after taking the action suggested by model. However, there remains a question that, the newly generated expression actually has many equivalent forms. In Table \ref{tab:vs_baselines}, we keep the terms' order consistent with the default order of Sympy, and then add several zeros at the end to form 8 terms. We call this order the default order, which guarantees that the proof given by the model will not change every time we use it. However, it also introduces a preference that valid actions tend to be at the front of a expression. Hence, we choose to shuffle the terms of generated expressions during the proof process and still test AutoTrig on $D_t$. We repeat 10 times and get the average performance of AutoTrig, as is shown in Table \ref{tab:generalization_task}. We can find AutoTrig still performs well, obviously better than rBFS. This shows that even if the terms of expressions are shuffled, AutoTrig can still understand the expression and predict the correct action.

\paragraph{Unseen coefficients (T2)}
When we generate the dataset, the coefficient of $x$ tends to be small, no more than 20. Hence, we introduce a linear transformation ($x \leftarrow a*x+b$) to extend the range of coefficients and get a new dataset with coefficients in the range from 1 to 360. In the generalization task T2, we apply different methods on the new dataset, and their performance is shown in Table \ref{tab:generalization_task}. The pass rate of AutoTrig 
drops slightly, but remains above 94\%. This shows that even if there are some coefficients in the expression that AutoTrig has never seen before, it can still understand the expression and predict the correct action. 

\paragraph{Real data (T3)}
Considering that artificially generated data may have a bias, in the generalization task T3, we collect several problems from the book \textit{103 Trigonometry Problems From the Training of the USA IMO Team} \cite{andreescu2004103}.In this book, many problems can not be solved by AutoTrig because they are not theorem proving problems or they can not be transformed to a normalized form defined in \textbf{Preliminary}. After some pre-processing with details shown in Appendix E, we finally get 10 problems. Then, we apply different methods on these problems, and their performance is shown in Table \ref{tab:generalization_task}. The result shows that AutoTrig can still solve these real problems and has advantages over other methods. This implies that the abilities AutoTrig learns from synthetic $D_t$ can be effectively used to solve real problems. To be more specific, in the Appendix E, we present proofs of a problem with different methods to show the power of AutoTrig.

% According to average length, we can find the real data is less complex than the data we generate. This means theses problems are easier to be proved, which can be reflected by the high pass rate. And we are happy to find AutoTrig outperform LBFS once again, with same pass rate, shorter average length and only 1/1000 of the average time. This task proves that the models we train with artificial data are also effective at solving real-world mathematical problems, showing the superiority of AutoTrig.
\section{Discussion}
We have successfully demonstrated the superiority of AutoTrig. However, it still has some shortcomings, which we have discussed in detail, as shown below.

\paragraph{More than $8$ addition terms}
Due to the limited input length of the BERT network, we previously introduce a constraint that an expression should have no more than 8 terms. With some tweaking, we can still use AutoTrig to prove these long identities. When the expression has more than 8 terms, we can use AutoTrig to deal with its first 8 terms. Repeat the processes above and we can finsh the proof.

\paragraph{More than one variable}
For brevity, we previously introduced a restriction that there is only one variable $x$ in the identity. In fact, there can be any number of variables in a trig identity. If its structure doesn't change, it is still possible to prove it by AutoTrig.
In the Appedix F, we show a case to prove a identity with multiple variables with rules in Table \ref{tab:rule}, after a tweaking in rule $A_{xx}$ . 

\paragraph{Restriction of rule set}
AutoTrig can solve a problem only if the problem can be solved with the rules in the Table 1. However, there are some identities need some extra rules to solve it, which makes it impossible for AutoTrig to solve it alone. As is shown in the Appendix F, we display two problems which need extra rules. Hence, more work need to be done to expand the rule set.

\paragraph{Gap between generated problems and real problems}Some real problems can be more complicated, and some involve extra conditions or summation of series, which introduces difficulty for our system to describe and solve these problems of novel forms.

% The first problem requires a more complex and flexible strategy to split the elements in a trig function. In Table \ref{tab:rule}, we only take the simplest strategy and we only split variables from constants. This strategy is simple but valid in most situations, but it runs into problems in proving the first problem. The second problem requires us to do operation in reverse, such as multiply by a factor at first. This operation is difficult even for a real person. In fact, this problem comes from the fifth IMO Competition in 1963. It is a challenge for AutoTrig to solve this problem.

% Here is an example, $sin(x+46*pi/90)-cos(x+pi/90)$. We can not solve it with the rules in Table 1. In this case, we need to split $x+46*pi/90$ in the form of $(x+pi/90)+pi/2$ rather than $(x) + 46/90$. This means we need too add more complex split rules in the rule set, bringing new challenge to AutoTrig. 
% In some situation, we even need to do some reverse operation.

\section{Conclusion}

In this work, we construct an automatic proof system for trigonometric identities, including defining the normalized form of trigonometric identities, designing the rule set, generating the identities and their proof. In order to achieve a high pass rate of proof and the shortest possible proof steps, we train the model from proof data generated by rBFS. We first prove theoretically that a model trained in this way can outperform rBFS, and then prove this experimentally. This also helps in other theorem-proving work. Through reinforcement learning, the model further improves and we get the final model AutoTrig. AutoTrig can prove trigonometric identities almost as short as BFS, but faster than BFS over a thousand times. User study points out that AutoTrig has successfully learned a mathematical intuition, which helps to pick proper actions to merge similar terms and take use of the common factors, leading to a short proof. Compared with Sympy and Matlab, the mainstream symbol computing processing tools in the market, AutoTrig achieves a better pass rate, showing its superiority. In addition, excellent performance in several generalization tasks proves that AutoTrig has strong generalization ability. At last, we have also discussed its limitations and shortcomings in several cases, which are left for future work.

\bibliographystyle{plain}
\bibliography{references}

\appendix

\section{Synthetic dataset}
\subsection{Synthetic identities}
The full details of generating identities are as follows:
\begin{enumerate}
    \item Randomly generate $n$ elements in the form of $\sin(A*x + B)$ or $\cos(A*x + B)$. Multiply these $n$ elements and a  coefficient $C$ to get a term in the form of $C*\prod_{i}^n trig_i(A_i*x+B_i)$. $n$ is randomly picked from $\{1, 2, 3, 4\}$. A is randomly picked from  $\{0, 1, 2, 3, 4, 5, 6\}$. B is randomly picked from $\{\pm\pi/2, \pm\pi/3, \pm\pi/4, \pm\pi/6, 0\}$. C is randomly picked from $\{0, \pm1, \pm2, \pm3, \pm4\}$. $trig_i$ is randomly picked from $\{\sin,\cos\}$.
    \item Repeat the step 1 for $m$ times and get $m$ terms. Add these terms to get a sum denoted by $E_0$. m is randomly picked from $\{1,2,3\}$.
    \item Randomly transform the current trigonometric expression $E_s$ and get a new equivalent expression $E_{s+1}$. The transform action includes two types. The first type is to deal with the product of two factors in one term in the expression, like $trig_1(A_1*x+B_1)*trig_2(A_2*x+B2)$.  For the product of two factors, we apply the corresponding P-rule on it and we can get a new expression. The second type is to deal with one factor in one term in the expression, like $trig(A*x+B)$. For the single factor, we need to divide $A*x+B$ into $(A-a)*x+(B-b)$ and $a*x+b$, where a is randomly picked from $\{0, 1, 2, 3, 4, 5, 6\}$ and b is randomly picked from $\{0,\pm\pi/2, \pm\pi/3, \pm\pi/4, \pm\pi/6\}$. Then we apply the corresponding A-rule on $trig(((A-a)*x+B-b))+(a*x+b))$ to get a new expression.
    
    \item Repeat the step 3 for $t$ times and get an expression $E_t$ which is equivalent to $E_0$. Thus, by subtracting $E_0$ from $E_t$, we can obtain a nontrivial trigonometric identity $E_t - E_0 = 0$. $t$ is randomly picked from $\{2, 3, 4, 5, 6\}$.
    
    \item Get rid of all the obvious extra factors in $E_t - E_0$. E.g. $\sin(x+\pi/3)*\cos(x) - \cos(x-\pi/6)*\cos(x)$ will be transformed into $\sin(x+\pi/3)-\cos(x-\pi/6)$. 
\end{enumerate}
In Generation 1, we display the process of generating a trigonometric identity using this method. In this case, we take 5 transformation actions to generate the final identity. However, we only need 3 steps to prove the identity equivalent to zero, shown in Proof 1. Hence, the reverse transformation sequence is not a good proof because it may take a long detour. In Generation 2, we show several identities provided by this generation method.

\setcounter{table}{0}
\renewcommand\tablename{Generation}
\begin{table}\small
    \centering
    \caption{The process to generate a trigonometric identity}
    \label{generation:1}
    \begin{tabular}{ll}
    \hline\hline
        $Step_1$: &Generate $E_0 = 2*\sin(3*x)*\sin(2*x+\pi/3)$\\\\
        
        $Step_2$:  &Type-A transformation, divide $3x$ in $\sin(3*x)$ into the sum of $3*x+\pi/3$ and $-\pi/3$,\\
        &get $E_1=\sin(2*x+\pi/3)*\sin(3*x+\pi/3)-sqrt(3)*\sin(2*x+\pi/3)*\cos(3*x+\pi/3)$\\\\
       
        $Step_3$: &Type-P transformation, deal with $\sin(2*x+\pi/3)*\cos(3*x+\pi/3)$,\\
        &get $E_2=sqrt(3)*\sin(x)/2+\sin(2*x+\pi/3)*\sin(3*x+\pi/3)-$\\&$sqrt(3)*\sin(5*x+2*\pi/3)/2$\\\\
        
        $Step_4$: &Type-P transformation, deal with $\sin(2*x+\pi/3)*\sin(3*x+\pi/3)$,\\
        &get $E_3=sqrt(3)*\sin(x)/2-\cos(5*x+2*\pi/3)/2+\cos(x)/2-$\\
        &$sqrt(3)*\sin(5*x+2*\pi/3)/2$\\\\
        
        $Step_5$: &Type-A transformation, divide $5*x+2*\pi/3$ in $\cos(5*x+2*\pi/3)$ into $5*x$ and $2*\pi/3$,\\
        &get $E_4=sqrt(3)*\sin(x)/2+sqrt(3)*\sin(5*x)/4+\cos(5*x)/4+$\\
        &   $\cos(x)/2-sqrt(3)*\sin(5*x+2*\pi/3)/2$\\\\
        
        $Step_6$: &Type-A transformation, divide $5*x+2*\pi/3$ in $\sin(5*x+2*\pi/3)$ into  $5*x$ and $2*\pi/3$,\\
        &get $E_5=sqrt(3)*\sin(x)/2+sqrt(3)*\sin(5*x)/2+\cos(x)/2-\cos(5*x)/2$\\\\
        
        $Step_7$: &Subtract $E_0$ from $E_5$ to get an identity: $sqrt(3)*\sin(x)/2+sqrt(3)*\sin(5*x)/2+$\\
        &$\cos(x)/2-\cos(5*x)/2-2*\sin(3*x)*\sin(2*x+\pi/3)$\\
    \hline\hline
    \end{tabular}
\end{table}

\subsection{Synthetic proof}
After generating a lot of identities, we design a filter to abandon cases with more than 8 items or repeated cases. Then we use rBFS to generate proofs for these identities. A proof can be recorded as $\{(s_0,a_0),...,(s_n,a_n), 0\}$. Here, we design a coding strategy which can use three numbers $i,j,k$ to describe an action taken to transform the current state. $i \in \{0,1,2,3,4,5,6,7\}$, $j \in \{0,1,2,3\}$, $k \in \{-1,0,1,2,3\}$. If $k$ is not equal to -1, $(i, j, k)$ means that we use P-rule to deal with the product of the $j$th and the $k$th factor in the $i$th term of $expr$. If $k$ is equal to -1, $(i, j, k)$ means that we use A-rule to deal with a singe factor, which is the $j$th factor in the $i$th term of $expr$. Swapping $j$ and $k$ has no effect on the actual operation while we deal with a product of two factors, so we set $j>=k$, $(j,k) \in\{(0,-1), (1,-1), (2,-1), (3,-1),(0,0),(1,0),(2,0),(3,0),(1,1),(2,1),(3,1),(2,2),(3,2),(3,3)\}$. Therefore, there are at most 112 actions for each state and we can use a number from 1 to 112 as the label of action. Then, we can use the proof data to train a classification task. In order to make the number of samples for each class similar, for each pair of ($s_i$,$a_i$):
\begin{itemize}
    \item Add extra $0$ to make $s_i$ have 8 terms; 
    \item Randomly shuffle the order of the terms in $s_i$ to get an equivalent $s_i^*$
    \item According to the shuffled $s_i^*$, change $a_i$ to $a_i^*$, and we get a new pair: ($s_i^*$, $a_i^*$)
\end{itemize}

\setcounter{table}{0}
\renewcommand\tablename{Proof}
\begin{table}[ht]\small
    \centering
    \caption{The proof of the generated identity}
    \begin{tabular}{ll}
    \hline\hline
    $S_0$:    & $sqrt(3)*\sin(x)/2+sqrt(3)*\sin(5*x)/2+\cos(x)/2-\cos(5*x)/2-$\\
    &$2*{\color{red}\sin(3*x)*\sin(2*x+\pi/3)}$ \\
    $a_0$:    & rule $P_{ss}$ on ${\color{red}\sin(3*x)*\sin(2*x+\pi/3)}$ \\\\
    
    $S_1$:    & $sqrt(3)*\sin(x)/2+sqrt(3)*\sin(5*x)/2+\cos(x)/2-\cos(5*x)/2-{\color{red}\cos(x-\pi/3)}+$\\
    &$\cos(5*x+\pi/3)$ \\
    $a_1$:    & rule $A_{c-}$ on ${\color{red}\cos(x-\pi/3)}$ \\\\
    
    $S_2$:    & $sqrt(3)*\sin(5*x)/2-\cos(5*x)/2-{\color{red}\cos(5*x+\pi/3)}$ \\
    $a_2$:    & rule $A_{c+}$ on ${\color{red}\cos(5*x+\pi/3)}$ \\\\
    
    $S_3$:    &$0$\\
    \hline\hline
    \end{tabular}
\end{table}

\setcounter{table}{1}
\renewcommand\tablename{Generation}
\begin{table}\small
    \centering
    \caption{Several cases of generated identities}
    \begin{tabular}{ll}
    \hline\hline
    $Case_0$:    & $2*\sin(2*x)*\sin(4*x + \pi/4) - \cos(2*x + \pi/4) + \cos(6*x + \pi/4)$\\\\
    $Case_1$:    &$-\sin(x)*\sin(6*x)*\cos(4*x + \pi/4) + \sin(x + \pi/4)/4 - \cos(7*x)*\cos(4*x + \pi/4)/2 + $\\
    &$\cos(9*x + \pi/4)/4$ \\\\
    $Case_2$:    & $-\sin(x)*\sin(4*x)*\cos(x + \pi/6) - \sin(4*x)*\sin(x + \pi/6)*\cos(x) + \sin(2*x + \pi/3)/2 -$\\
    &$\cos(6*x + \pi/6)/2$\\\\
    $Case_3$:    & $sqrt(3)*\sin(x +\pi/3)/4 - \sin(4*x + \pi/3)*\sin(5*x + \pi/3) + \sin(9*x + \pi/6)/2 +$\\
    &$\cos(x + \pi/3)/4$\\\\
    $Case_4$:    &$-\sin(5*x)/2 + \sin(x + \pi/3)/2 - \cos(2*x + \pi/3)*\cos(3*x + \pi/6)$\\
    \hline\hline
    \end{tabular}
\end{table}

\section{Proof of Theorem 1}
\begin{proof} 
A node of trigonometric expression can appear in different proof trees corresponding to different trigonometric identity problems. Thus, we can rewrite \eqref{eq:loss} in the following,
\begin{align}
    & \min_\theta - \sum_{s \in \S} \sum_{a \in \A_s} \text{Pr}(\text{Event}(s, a)) \log P_\theta(a \mid s) \\
    & s.t. ~ \sum_{a \in \A_s} P_\theta(a \mid s) = 1 \text{ for all } s,
\end{align}
where $\S$ is the set of expressions, $\A_s$ is the set of transformation action at $s$, $\text{Event}(s, a)$ denotes the event that action $a$ at expression $s$ corresponds to the shortest proof path among three branches selected by rBFS. Note that the above problem has no dependency between different $s$, it can be split into $|\S|$ sub-problems for each $s$,
\begin{align}
    & \min_P - \sum_{a \in \A_s} \text{Pr}(\text{Event}(s, a)) \log P_\theta(a \mid s) \\
    & s.t. ~ \sum_{a \in \A_s} P_\theta(a \mid s) = 1.
\end{align}

% \begin{align}
%     &   P_\D(a_i \mid s) = \frac{C_{n-i}^2}{C_n^3}
% \end{align}

According to Lemma \ref{lem}, we can get the optimal solution 
\begin{align}
    P_{\theta^*}(a \mid s) = \text{Pr}(\text{Event}(s, a)).
\end{align}
Consider calculating $\text{Pr}(\text{Event}(s, a)$. Suppose that node $s$ has $n$ branches, we sort the $n$ branches according to $L(s, a_1), L(s, a_2), \ldots, L(s, a_n)$ in an ascend order. Then, we reorder $\{a_1, a_2, \ldots, a_n\}$ such that 
\begin{align}
    L(s, a_1) \leq L(s, a_2) \leq \ldots \leq L(s, a_n).
\end{align}
As for $a_i$, we have that 
\begin{align}
    \text{Pr}(\text{Event}(s, a_i)) = \frac{C_{n-i}^2}{C_n^3}.
\end{align}
Let $C_{n-i}^2 = 0$ when $n - i < 2$. Thus, we get that $\arg\max_a P_{\theta^*}(a \mid s) = \arg\min_a L(s, a)$.
\end{proof}

\begin{lemma}\label{lem}
Let $p_i > 0$ for $i \in \{1, 2, \ldots, n\}$, $\sum_i p_i = 1$, $x_i \geq 0$ for $i \in \{1, 2, \ldots, n\}$ and $\sum_i x_i = 1$. We have that
\begin{align} 
\sum_{i=1}^n p_i \log x_i \leq \sum_{i=1}^n p_i \log p_i. 
\end{align}
\end{lemma}

\begin{proof} 
With the weighted arithmetic mean–geometric mean inequality, we have that 
\begin{align} 
1 = \sum_{i=1}^n p_i \frac{x_i}{p_i} \geq \prod_{i=1}^n \left(\frac{x_i}{p_i}\right)^{p_i}.
\end{align}
Take logarithm on both sides, we have that 
\begin{align} 
0 \geq \sum_{i=1}^n p_i \log \left(\frac{x_i}{p_i}\right) = \sum_{i=1}^n p_i \log x_i - \sum_{i=1}^n p_i \log p_i. 
\end{align} 
Thus, we get the proof. The equality holds if and only if $\frac{x_1}{p_1} = \frac{x_2}{p_2} = \ldots = \frac{x_n}{p_n}$.
\end{proof}

\section{Markov decision process}
As presented in Figure \ref{fig:proof_tree}, this problem can be viewed as a search problem. Thus, it can be modeled by a Markov decision process $M = (\S, \A, P, R, \gamma)$. $\S$ is the set of trigonometric expressions which are in the form of \eqref{eq:trigonometric_expression} and can be simplified to zero. It corresponds to a node in the search tree. $\A$ is the set of transformation actions which correspond to branches in Figure \ref{fig:proof_tree}. $P$ is the transition probability. In our problem, it is a deterministic transformation. $R$ is the reward function that gives the reward $R(s)$ at node $s$. The reward is $1$ when we finish the proof, $-0.1$ otherwise. The goal is to maximize the cumulative reward, $\max_\pi \E_{s_0, s_1, \ldots \sim \pi} \sum_{i} \gamma^i R(s_i)$ where $\pi$ is the action selection policy, $s_0, s_1, \ldots$ is the proof path obtained by $\pi$.

\section{User study}
Here we show the details of proof for \textbf{Problem 1}, with different methods (BFS, AutoTrig, rBFS and volunteer). We can find that AutoTrig can prove the problem within 5 steps, as short as BFS, while rBFS costs 9 steps and the volunteer costs 10 steps. BFS can find the shortest proof by searching all possible paths, but it costs the most time. But AutoTrig can finish this task with the 
least time cost and the shortest length at the same time. The key is that AutorTrig can pick a proper action which is helpful to merge similar terms to decrease the number of terms. Moreover, this ability is not short-sighted. We can find that after the first operation, the number of expression's terms increases from 4 to 5, but after the second operation, the number is shortened to 3.  This suggests that AutoTrig has learned some intuition from training to select the appropriate operation. But for rBFS and volunteer, it is not always easy to choose the right action from so many possible actions. Once they take a bad operation, they will get a very complicated formula, which makes it more difficult to do the next choice, finally leading to a very long proof.

\setcounter{table}{0}
\renewcommand\tablename{Problem}
\begin{table}[h]\footnotesize
    \centering
    \caption{}
    \label{problem:case1}
    \begin{tabular}{l}
    \\
    $-\sin(2*x)*\cos(x + \pi/6)-\sin(x+\pi/6)**2*\cos(x +\pi/3)-$\\
    $\sin(x+\pi/6)*\sin(x + \pi/3)*\cos(x+ \pi/6) + \sin(3*x + \pi/6)$\\
    \end{tabular}
\end{table}

%列出一道题 展示BFS LBFS 模型 以及人的操作的区别 可以走近路
\setcounter{table}{1}
\renewcommand\tablename{Proof}
\begin{table}\footnotesize
    \centering
    \caption{BFS's Proof for Problem 1   (Cost  1273 sec)}
    \begin{tabular}{ll}
    \hline\hline
        $S_0$:& $-\sin(2*x)*\cos(x + \pi/6)-{\color{red}\sin(x+\pi/6)}**2*{\color{red}\cos(x +\pi/3)} -$\\
        &$\sin(x+\pi/6)*\sin(x + \pi/3)*\cos(x+ \pi/6) + \sin(3*x + \pi/6)$\\
        $a_0$: &rule $P_{sc}$ on ${\color{red}\sin(x+\pi/6)*\cos(x+\pi/3)}$ \\\\

        $S_1$: & $-{\color{red}\sin(2*x)*\cos(x + \pi/6)}- \sin(x + \pi/6)*\sin(x + \pi/3)*\cos(x + \pi/6)-$\\
        &$\sin(x + \pi/6)*\sin(2*x + \pi/2)/2+\sin(x + \pi/6)/4+\sin(3*x + \pi/6)$ \\
        $a_1$: &rule $P_{sc}$ on ${\color{red}\sin(2*x)*\cos(x + \pi/6)}$ \\\\
        
        $S_2$: & $-\sin(x-\pi/6)/2-\sin(x+\pi/6)*{\color{red}\sin(x+\pi/3)*\cos(x+\pi/6)}-$\\
        &$\sin(x+\pi/6)*\sin(2*x+\pi/2)/2+\sin(x+\pi/6)/4+\sin(3*x+\pi/6)/2$ \\
        $a_{2}$: &rule $P_{sc}$ on ${\color{red}\sin(x+\pi/3)*\cos(x+\pi/6)}$ \\\\

        $S_3$: & $-\sin(x-\pi/6)/2-\sin(x+\pi/6)*{\color{red}\sin(2*x+\pi/2)}+\sin(3*x+\pi/6)/2$ \\
        $a_3$: &rule $A_{s+}$ on ${\color{red}\sin(2*x+\pi/2)}$ \\\\
        
        $S_4$: & $-\sin(x-\pi/6)/2-{\color{red}\sin(x+\pi/6)*\cos(2*x)}+\sin(3*x+\pi/6)/2$ \\
        $a_{4}$: &rule $P_{sc}$ on ${\color{red}\sin(x+\pi/6)*\cos(2*x)}$ \\\\
        
        $S_5$: &$0$\\
        \hline\hline
    \end{tabular}
    % \caption{Caption}
\end{table}

\begin{table}\footnotesize
    \centering
    \caption{AutoTrig's Proof for Problem 1    (Cost 0.28 sec)}
    \begin{tabular}{ll}
    \hline\hline
        $S_0$:& $-\sin(2*x)*\cos(x + \pi/6)-\sin(x+\pi/6)**2*\cos(x +\pi/3) -$\\
        &$\sin(x+\pi/6)*{\color{red}\sin(x + \pi/3)*\cos(x+ \pi/6)} + \sin(3*x + \pi/6)$\\
        $a_0$: &rule $P_{sc}$ on ${\color{red}\sin(x + \pi/3)*\cos(x+ \pi/6)}$ \\\\

        $S_1$: & $-\sin(2*x)*\cos(x + \pi/6)-{\color{red}\sin(x+\pi/6)}**2*{\color{red}\cos(x +\pi/3)}-$\\
        &$\sin(x + \pi/6)*\sin(2*x + \pi/2)/2-\sin(x + \pi/6)/4+\sin(3*x + \pi/6)$ \\
        $a_1$: &rule $P_{sc}$ on ${\color{red}\sin(x+\pi/6)*\cos(x+\pi/3)}$ \\\\
        
        $S_2$: & $-\sin(2*x)*\cos(x + \pi/6)-\sin(x+\pi/6)*{\color{red}\sin(2*x+\pi/2)}+\sin(3*x+\pi/6)$ \\
        $a_{2}$: &rule $A_{s+}$ on ${\color{red}\sin(2*x+\pi/2)}$ \\\\

        $S_3$: & $-\sin(2*x)*\cos(x + \pi/6)-{\color{red}\sin(x+\pi/6)*\cos(2*x)}+\sin(3*x+\pi/6)$ \\
        $a_3$: &rule $P_{sc}$ on ${\color{red}\sin(x+\pi/6)*\cos(2*x)}$ \\\\
        
        $S_4$: & $-{\color{red}\sin(2*x)*\cos(x + \pi/6)}+\sin(x-\pi/6)/2 +\sin(3*x+\pi/6)/2$ \\
        $a_{4}$: &rule $P_{sc}$ on ${\color{red}\sin(2*x)*\cos(x + \pi/6)}$ \\\\
        
        $S_5$: &$0$\\
        \hline\hline
    \end{tabular}
    % \caption{Caption}
\end{table}

\begin{table}\footnotesize
    \centering
    \caption{rBFS's Proof for Problem 1    (Cost 470 sec)}
    \begin{tabular}{ll}
    \hline\hline
        $S_0$:& $-\sin(2*x)*\cos(x + \pi/6)-{\color{red}\sin(x+\pi/6)**2}*\cos(x +\pi/3) -$\\
        &$\sin(x+\pi/6)*\sin(x + \pi/3)*\cos(x+ \pi/6) + \sin(3*x + \pi/6)$\\
        $a_0$: &rule $P_{ss}$ on ${\color{red}\sin(x+\pi/6)**2}$ \\\\

        $S_1$: & $-\sin(2*x)*\cos(x + \pi/6) -\sin(x+\pi/6)*\sin(x + \pi/3)*\cos(x+ \pi/6)+$\\
        &$\sin(3*x + \pi/6)-\cos(x+\pi/3)/2+{\color{red}\cos(x+\pi/3)*\cos(2*x+\pi/3)}/2$ \\
        $a_1$: &rule $P_{cc}$ on ${\color{red}\cos(x+\pi/3)*\cos(2*x+\pi/3)}$ \\\\
        
        $S_2$: & $-\sin(2*x)*\cos(x + \pi/6) -{\color{red}\sin(x+\pi/6)}*\sin(x + \pi/3)*{\color{red}\cos(x+ \pi/6)}+$\\
        &$\sin(3*x + \pi/6)+\cos(x)/4-\cos(x+\pi/3)/2+\cos(3*x+2*\pi/3)/4$ \\
        $a_{2}$: &rule $P_{sc}$ on ${\color{red}\sin(x+\pi/6)*\cos(x+\pi/6)}$ \\\\

        $S_3$: & $-{\color{red}\sin(2*x)*\cos(x + \pi/6)} -\sin(x + \pi/3)*\sin(2*x+\pi/3)/2+$\\
        &$\sin(3*x + \pi/6)+\cos(x)/4-\cos(x+\pi/3)/2+\cos(3*x+2*\pi/3)/4$ \\
        $a_3$: &rule $P_{sc}$ on ${\color{red}\sin(2*x)*\cos(x + \pi/6)}$ \\\\
        
        $S_4$: & $-\sin(x-\pi/6)/2 -\sin(x + \pi/3)*\sin(2*x+\pi/3)/2+\sin(3*x + \pi/6)/2$\\
        &$+\cos(x)/4-{\color{red}\cos(x+\pi/3)}/2+\cos(3*x+2*\pi/3)/4$\\
        $a_4$: &rule $A_{c+}$ on ${\color{red}\cos(x+\pi/3)}$ \\\\
        
        $S_5$: & $sqrt(3)*\sin(x)/4-\sin(x-\pi/6)/2 -{\color{red}\sin(2*x+\pi/3)*\sin(x + \pi/3)}/2+$\\
        &$\sin(3*x + \pi/6)/2+\cos(3*x+2*\pi/3)/4$\\
        $a_5$: &rule $P_{ss}$ on ${\color{red}\sin(2*x+\pi/3)*\sin(x + \pi/3)}$ \\\\
        
        $S_6$: & $sqrt(3)*\sin(x)/4-\sin(x-\pi/6)/2+{\color{red}\sin(3*x+\pi/6)}/2-\cos(x)/4+$\\
        &$\cos(3*x+2*\pi/3)/2$\\
        $a_6$: &rule $A_{s+}$ on ${\color{red}\sin(3*x+\pi/6)}$ \\\\
        
        $S_7$: & $sqrt(3)*\sin(x)/4+sqrt(3)*\sin(3*x)/4-\sin(x-\pi/6)/2-\cos(x)/4+$\\
        &$\cos(3*x)/4+{\color{red}\cos(3*x+2*\pi/3)}/2$\\
        $a_7$: &rule $A_{c+}$ on ${\color{red}\cos(3*x+2*\pi/3)}$ \\\\

        $S_8$: & $sqrt(3)*\sin(x)/4-{\color{red}\sin(x-\pi/6)}/2-\cos(x)/4$\\
        $a_8$: &rule $A_{s-}$ on ${\color{red}\sin(x-\pi/6)}$ \\\\

        $S_9$: &$0$\\

        \hline\hline
    \end{tabular}
    % \caption{Caption}
\end{table}

\begin{table}\footnotesize
    \centering
    \caption{Volunteer's Proof for Problem 1    (Cost 153 sec)}
    \begin{tabular}{ll}
    \hline\hline
        $S_{0}$:& $-\sin(2*x)*\cos(x + \pi/6)-{\color{red}\sin(x+\pi/6)}**2*{\color{red}\cos(x +\pi/3)} -$\\
        &$\sin(x+\pi/6)*\sin(x + \pi/3)*\cos(x+ \pi/6) + \sin(3*x + \pi/6)$\\
        $a_{0}$: &rule $P_{sc}$ on ${\color{red}\sin(x+\pi/6)*\cos(x+\pi/3)}$ \\\\
        
        $S_{1}$: & $-\sin(2*x)*\cos(x + \pi/6)- {\color{red}\sin(x + \pi/6)}*\sin(x + \pi/3)*{\color{red}\cos(x + \pi/6)}-$\\
        &$\sin(x + \pi/6)*\sin(2*x + \pi/2)/2+\sin(x + \pi/6)/4+\sin(3*x + \pi/6)$ \\
        $a_{1}$: &rule $P_{sc}$ on ${\color{red}\sin(x+\pi/6)*\cos(x+\pi/6)}$ \\\\
        
        $S_{2}$: & $-\sin(2*x)*\cos(x + \pi/6)- \sin(x + \pi/6)*{\color{red}\sin(2*x + \pi/2)}/2+$\\
        &$\sin(x +  \pi/6)/4- \sin(x + \pi/3)*\sin(2*x + \pi/3)/2+\sin(3*x + \pi/6)$ \\
        $a_{2}$: &rule $A_{s+}$ on ${\color{red}\sin(2*x + \pi/2)}$ \\\\
        
        $S_{3}$: & $-\sin(2*x)*\cos(x + \pi/6)- \sin(x + \pi/6)*\cos(2*x)/2+$\\
        &$\sin(x + \pi/6)/4- {\color{red}\sin(x + \pi/3)*\sin(2*x + \pi/3)}/2+\sin(3*x + \pi/6)$ \\
        $a_{3}$: &rule $P_{ss}$ on ${\color{red}\sin(x + \pi/3)*\sin(2*x + \pi/3)}$ \\\\
        
        $S_{4}$: & $-{\color{red}\sin(2*x)*\cos(x + \pi/6)}- \sin(x + \pi/6)*\cos(2*x)/2+$\\
        &$\sin(x + \pi/6)/4+\sin(3*x + \pi/6)- \cos(x)/4+\cos(3*x + 2*\pi/3)/4$ \\
        $a_{4}$: &rule $P_{sc}$ on ${\color{red}\sin(2*x)*\cos(x + \pi/6)}$ \\\\
        
        $S_{5}$: & $-\sin(x - \pi/6)/2- {\color{red}\sin(x + \pi/6)*\cos(2*x)}/2+\sin(x + \pi/6)/4+$\\
        &$\sin(3*x + \pi/6)/2- \cos(x)/4+\cos(3*x + 2*\pi/3)/4$ \\
        $a_{5}$: &rule $P_{sc}$ on ${\color{red}\sin(2*x)*\cos(x + \pi/6)}$ \\\\
        
        $S_{6}$: & $-{\color{red}\sin(x - \pi/6)}/4+\sin(x + \pi/6)/4+\sin(3*x + \pi/6)/4-\cos(x)/4+$\\ 
        &$\cos(3*x + 2*\pi/3)/4$ \\
        $a_{6}$: &rule $A_{s-}$ on ${\color{red}\sin(x - \pi/6)}$ \\\\
        
        $S_{7}$: & $-sqrt(3)*\sin(x)/8+{\color{red}\sin(x + \pi/6)}/4+\sin(3*x + \pi/6)/4-$\\ &$\cos(x)/8+\cos(3*x + 2*\pi/3)/4$ \\
        $a_{7}$: &rule $A_{s+}$ on ${\color{red}\sin(x + \pi/6)}$ \\\\

        $S_{8}$: & ${\color{red}\sin(3*x+\pi/6)}/4+\cos(3*x + 2*\pi/3)/4$ \\
        $a_{8}$: &rule $A_{s+}$ on ${\color{red}\sin(3*x+\pi/6)}$ \\\\
        
        $S_{9}$: & $-sqrt(3)*\sin(3*x)/8+\cos(3*x)/8+{\color{red}\cos(3*x + 2*\pi/3)}/4$ \\
        $a_{9}$: &rule $A_{c+}$ on ${\color{red}\cos(3*x + 2*\pi/3)}$ \\\\
        
        $S_{10}$: &$0$
        
        \\ \hline\hline
    \end{tabular}
    % \caption{Caption}
\end{table}

\section{Real data}
%列出一道真题 例如tan那道题，然后讲讲我们对真题的预处理
Considering that artificially generated data may have a bias, we collect some problems from the book \textit{103 Trigonometry Problems From the Training of the USA IMO Team}. In order for the problems to meet the requirements of our model, we have to do some pre-processing on them, including:
\begin{itemize}
    \item Use $\sin$ and $\cos$ to replace $\tan$, $\cot$, $\sec$, $\csc$.
    \item Eliminate the denominator.
    \item Replace the independent variable in the problem with $x$.
\end{itemize}
Here, we show an example, which is the $Introductory$ $Problems$ $13$ in the book. After pre-processing, we can get its equivalent form (\textbf{Problem 2}). And we also show the details of proof for Problem 2, with different methods (BFS, AutoTrig and rBFS). Both BFS and AutoTrig can solve the problem within 4 steps, while rBFS needs 6 steps. We can find that the proof of both BFS and AutoTrig can get an expression with a common factor. The common factor is $\sin(3*x)$ in BFS's proof, and there is a common factor $\cos(3*x)$ in AutoTrig's proof. By extracting common factors, the proof can be shortened so that BFS and AutoTrig outperform rBFS. In the view of time cost, AutoTrig obviously outperforms BFS and rBFS. An interesting point is that BFS even costs less time than rBFS. This is probably because BFS finishes the proof in just four steps, so although it needs to search all branches at each step, it can terminate the task after the fourth search, leading to a even smaller search space than rBFS.

\setcounter{table}{1}
\renewcommand\tablename{Problem}
\begin{table}\footnotesize
    \centering
    \caption{}
    \begin{tabular}{ll}
    \\
    \textbf{Original problem}:    & $\tan(3*x)-\tan(2*x)-\tan(x) = \tan(3*x)*\tan(2*x)*\tan(x)$\\\\
    \textbf{After pre-processing}:  & $-\sin(x)*\sin(2*x)*\sin(3*x) - \sin(x)*\cos(2*x)*\cos(3*x)- $\\
    &$ \sin(2*x)*\cos(x)*\cos(3*x)+\sin(3*x)*\cos(x)*\cos(2*x)$\\
    \end{tabular}
\label{problem:case2}
\end{table}

\setcounter{table}{5}
\renewcommand\tablename{Proof}
\begin{table}\footnotesize
    \centering
    \caption{BFS's Proof for Problem 2    (Cost 14.5 sec)}
    \begin{tabular}{ll}
    \hline\hline
        $S_0$:&$-{\color{red}\sin(x)*\sin(2*x)}*\sin(3*x) - \sin(x)*\cos(2*x)*\cos(3*x)- $\\
    &$ \sin(2*x)*\cos(x)*\cos(3*x)+\sin(3*x)*\cos(x)*\cos(2*x)$\\
        $a_0$: &rule $P_{ss}$ on ${\color{red}\sin(x)*\sin(2*x)}$ \\\\

        $S_1$: & $- {\color{red}\sin(x)*\cos(2*x)}*\cos(3*x)- \sin(2*x)*\cos(x)*\cos(3*x)+$\\
    &$ \sin(3*x)*\cos(x)*\cos(2*x)-\sin(3*x)*\cos(x)/2+\sin(3*x)*\cos(3*x)/2$\\
        $a_1$: &rule $P_{sc}$ on ${\color{red}\sin(x)*\cos(2*x)}$ \\\\
        
        $S_2$: & $\sin(x)*\cos(3*x)/2- {\color{red}\sin(2*x)*\cos(x)}*\cos(3*x)+$\\
    &$ \sin(3*x)*\cos(x)*\cos(2*x)-\sin(3*x)*\cos(x)/2$\\
        $a_{2}$: &rule $P_{sc}$ on ${\color{red}\sin(2*x)*\cos(x)}$ \\\\

        $S_3$:&${\color{brown}\sin(3*x)}*{\color{red}\cos(x)*\cos(2*x)}-{\color{brown}\sin(3*x})*\cos(x)/2-{\color{brown}\sin(3*x)}*\cos(3*x)/2$\\
        $a_3$: &rule $P_{cc}$ on ${\color{red}\cos(x)*\cos(2*x)}$ \\\\

        $S_4$: &$0$\\
        \hline\hline
    \end{tabular}
    % \caption{Caption}
\end{table}

\begin{table}\footnotesize
    \centering
    \caption{AutoTrig's Proof for Problem  2    (Cost 0.2 sec)}
    \begin{tabular}{ll}
    \hline\hline
        $S_0$:&$-{\color{red}\sin(x)*\sin(2*x)}*\sin(3*x) - \sin(x)*\cos(2*x)*\cos(3*x)- $\\
    &$ \sin(2*x)*\cos(x)*\cos(3*x)+\sin(3*x)*\cos(x)*\cos(2*x)$\\
        $a_0$: &rule $P_{ss}$ on ${\color{red}\sin(x)*\sin(2*x)}$ \\\\

        $S_1$: & $- \sin(x)*\cos(2*x)*\cos(3*x)- {\color{red}\sin(2*x)*\cos(x)}*\cos(3*x)+$\\
    &$ \sin(3*x)*\cos(x)*\cos(2*x)-\sin(3*x)*\cos(x)/2+\sin(3*x)*\cos(3*x)/2$\\
        $a_1$: &rule $P_{sc}$ on ${\color{red}\sin(2*x)*\cos(x)}$ \\\\
        
        $S_2$: & $- \sin(x)*\cos(2*x)*\cos(3*x)-\sin(x)*\cos(3*x)/2+$\\
    &$ \sin(3*x)*{\color{red}\cos(x)*\cos(2*x)}-\sin(3*x)*\cos(x)/2$\\
        $a_{2}$: &rule $P_{cc}$ on ${\color{red}\cos(x)*\cos(2*x)}$ \\\\

        $S_3$: &$- {\color{red}\sin(x)*\cos(2*x)}*{\color{brown}\cos(3*x)}-\sin(x)*{\color{brown}\cos(3*x)}/2+\sin(3*x)*{\color{brown}\cos(3*x)}/2$\\
        $a_3$: &rule $P_{sc}$ on ${\color{red}\sin(x)*\cos(2*x)}$ \\\\

        $S_4$: &$0$\\
        \hline\hline
    \end{tabular}
    % \caption{Caption}
\end{table}

\begin{table}\footnotesize
    \centering
    \caption{rBFS's Proof for Problem 2    (Cost 15.1 sec)}
    \begin{tabular}{ll}
    \hline\hline
        $S_0$:&$-\sin(x)*\sin(2*x)*\sin(3*x) - \sin(x)*\cos(2*x)*\cos(3*x)- $\\
    &$ \sin(2*x)*\cos(x)*\cos(3*x)+{\color{red}\sin(3*x)*\cos(x)}*\cos(2*x)$\\
        $a_0$: &rule $P_{sc}$ on ${\color{red}\sin(3*x)*\cos(x)}$ \\\\

        $S_1$: & $-{\color{red}\sin(x)}*\sin(2*x)*{\color{red}\sin(3*x)} - \sin(x)*\cos(2*x)*\cos(3*x)- $\\
    &$ \sin(2*x)*\cos(x)*\cos(3*x)+\sin(2*x)*\cos(2*x)/2+\sin(4*x)*\cos(2*x)/2$\\
        $a_1$: &rule $P_{ss}$ on ${\color{red}\sin(x)*\sin(3*x)}$ \\\\
        
        $S_2$: & $- \sin(x)*\cos(2*x)*\cos(3*x)-\sin(2*x)*\cos(x)*\cos(3*x)+ $\\
    &${\color{red}\sin(2*x)*\cos(4*x)}/2+\sin(4*x)*\cos(2*x)/2$\\
        $a_{2}$: &rule $P_{sc}$ on ${\color{red}\sin(2*x)*\cos(4*x)}$ \\\\

        $S_3$: &$- {\color{red}\sin(x)}*\cos(2*x)*{\color{red}\cos(3*x)}-\sin(2*x)*\cos(x)*\cos(3*x)-$\\
    &$\sin(2*x)/4+\sin(4*x)*\cos(2*x)/2+\sin(6*x)/4$\\
        $a_3$: &rule $P_{sc}$ on ${\color{red}\sin(x)*\cos(3*x)}$ \\\\
        
        $S_4$: &$-\sin(2*x)*{\color{red}\cos(x)*\cos(3*x)}+\sin(2*x)*\cos(2*x)/2-$\\
    &$\sin(2*x)/4+\sin(6*x)/4$\\
        $a_{4}$: &rule $P_{cc}$ on ${\color{red}\cos(x)*\cos(3*x)}$ \\\\
        
        $S_5$: &$-{\color{red}\sin(2*x)*\cos(4*x)}/2-\sin(2*x)/4+\sin(6*x)/4$\\
        $a_{5}$: &rule $P_{sc}$ on ${\color{red}\sin(2*x)*\cos(4*x)}$ \\\\
        
        $S_6$: &$0$\\
        \hline\hline
    \end{tabular}
    % \caption{Caption}
\end{table}

\section{Discussion} 
In Proof 9, we show the proof of a problem with multiple variables. We can find that this problem can be proved with the rules in Table \ref{tab:rule}. This shows that AutoTrig can be used to solve multivariable trigonometric identities, with a little modification about the rule $A_{xx}$. When we apply the rule $A_{xx}$, we need to split all variables from constants. Then, we can use AutoTrig to solve trigonometric identities with multiple factors.  

In Proof 10, we show a case which needs a more complex and flexible strategy to divide the elements in a trig function. In Table \ref{tab:rule}, we only take the simplest strategy and we only split variables from constants. If we divide $x+46*\pi/90$ into $x$ and $46*\pi/90$ and divide $x+\pi/90$ into $x$ and $\pi/90$, we can not finish the proof. But if we divide $x+46*\pi/90$ into $x+\pi/90$ and $\pi/2$, the problem can be proved within only one step. Hence, a proper strategy for dividing terms is important. However, for an expression, we theoretically have an infinite number of ways to divide it. The key is that we need to find the hidden relationship between angles in the expression, so that we can make a proper split. This creates new challenges for AutoTrig and requires more work from us. In Proof 11, we show a case from the fifth IMO Competition in 1963, which even needs an extra factor $\sin(\pi/7)$ during the proving. It's a clever trick, but hard to come up with. The difficulty is that we don't know how to design the action space for such a flexible operation. Because we can theoretically multiply this expression by any non-zero factor. How to choose a proper factor remains a great challenge for AutoTrig.

%列出两道题
%
%例如tan那道题，然后讲讲我们对真题的预处理
\begin{table}\footnotesize
    \centering
    \caption{Special case1}
    \begin{tabular}{ll}
    \hline\hline
        $S_0$:&$\cos(2*x)+\cos(2*y)-\cos(2*x+2*y)-4*\sin(x)*\sin(y)*{\color{red}\sin(\pi/2-x-y)}-1$\\
        
        $a_0$: & divide $\pi/2-x-y$ in ${\color{red}\sin(\pi/2-x-y)}$ into the sum of $\pi/2$ and $-(x+y)$,\\
        &then apply rule $A_{s-}$ on it\\\\
        
        $S_1$: &$\cos(2*x)+\cos(2*y)-\cos(2*x+2*y)-4*{\color{red}\sin(x)*\sin(y)}*\cos(x+y)-1$\\
        $s_1$: &rule $P_{ss}$ on ${\color{red}\sin(x)*\sin(y)}$\\\\
        
        $S_2$: &$\cos(2*x)+\cos(2*y)-\cos(2*x+2*y)-2*{\color{red}\cos(x-y)*\cos(x+y)} +2*\cos(x+y)**2-1$\\
        $s_2$: &rule $P_{cc}$ on ${\color{red}\cos(x-y)*\cos(x+y)}$\\\\
        
        $S_3$: &$2*{\color{red}\cos(x+y)**2}-1$\\
        $a_3$: &rule $P_{cc}$ on ${\color{red}\cos(x+y)**2}$\\\\
        
        $S_4$: $0$\\
        \hline\hline
    \end{tabular}
    % \caption{Caption}
\end{table}

\begin{table}\footnotesize
    \centering
    \caption{Special case2}
    \begin{tabular}{ll}
    \hline\hline
        $S_0$:&${\color{red}\sin(x+46*\pi/90)}-\cos(x+\pi/90)$\\
        $a_0$: & divide $x+46*\pi/90$ in ${\color{red}\sin(x+46*\pi/90)}$ into the sum of $x+\pi/90$ and $\pi/2$,\\
        &then apply rule $A_{s+}$ on it\\\\
        $S_1$: &$0$\\
        \hline\hline
    \end{tabular}
    % \caption{Caption}
\end{table}

\begin{table}\footnotesize
    \centering
    \caption{Special case3}
    \begin{tabular}{ll}
    \hline\hline
        $S_0$:  &$-2*\cos(\pi/7)+2*\cos(2*\pi/7)-2*\cos(3*\pi/7)+1$\\
        $a_0$:&multiply $S_0$ by $sin(\pi/7)$\\\\
        
        $S_1$:&$-2*{\color{red}\cos(\pi/7)*\sin(\pi/7)}+2*\cos(2*\pi/7)*\sin(\pi/7)-2*\cos(3*\pi/7)*\sin(\pi/7)+\sin(\pi/7)$\\
        $a_1$:&rule $P_{cs}$ on ${\color{red}\cos(\pi/7)*\sin(\pi/7)}$\\\\
        
        $S_2$:&$-\sin(2*\pi/7)+2*{\color{red}\cos(2*\pi/7)*\sin(\pi/7)}-2*\cos(\pi/7)*\sin(\pi/7)+\sin(\pi/7)$\\
        $a_2$:&rule $P_{cs}$ on ${\color{red}\cos(2*\pi/7)*\sin(\pi/7)}$\\\\

        $S_3$:&$-\sin(2*\pi/7)+\sin(3*\pi/7)-\sin(\pi/7)-2*{\color{red}\cos(3*\pi/7)*\sin(\pi/7)}+\sin(\pi/7)$\\
        $a_3$:&rule $P_{cs}$ on ${\color{red}\cos(3*\pi/7)*\sin(\pi/7)}$\\\\
        
        $S_4$:&$\sin(3*\pi/7)-{\color{red}\sin(4*\pi/7)}$\\
        $a_4$: & divide $4*\pi/7$ in ${\color{red}\sin(4*\pi/7)}$ into the sum of $\pi$ and $-3*\pi/7$,\\
        &then apply rule $A_{s-}$ on $\sin(\pi-3*\pi/7)$\\\\
        
        $S_5$:&$0$\\
        \hline\hline
    \end{tabular}
    % \caption{Caption}
\end{table}

\end{document}